\documentclass[letterpaper]{article} % DO NOT CHANGE THIS
\usepackage{aaai2026}  % DO NOT CHANGE THIS
\usepackage{times}  % DO NOT CHANGE THIS
\usepackage{helvet}  % DO NOT CHANGE THIS
\usepackage{courier}  % DO NOT CHANGE THIS
\usepackage[hyphens]{url}  % DO NOT CHANGE THIS
\usepackage{graphicx} % DO NOT CHANGE THIS
\usepackage{natbib}  % DO NOT CHANGE THIS AND DO NOT ADD ANY OPTIONS TO IT
\usepackage{caption} % DO NOT CHANGE THIS AND DO NOT ADD ANY OPTIONS TO IT
\usepackage{lineno} % Added to fix undefined \@LN@col
\usepackage{algo}
\usepackage{algorithmic}

\usepackage{newfloat}
\usepackage{listings}
\DeclareCaptionStyle{ruled}{labelfont=normalfont,labelsep=colon,strut=off} % DO NOT CHANGE THIS
\usepackage{amssymb}
\usepackage{amsmath}
\usepackage{amsthm}
\usepackage{cleveref}

\renewcommand{\phi}{\varphi}

\newcommand{\median}{\operatorname{\textbf{M}}}
\RequirePackage{aliascnt}

\def\MAE{\textbf{MAE}}

\let\R\Real

\newtheorem{theorem}{Theorem}
\newtheorem{lemma}{Lemma}

\newaliascnt{corollary}{theorem}
\newaliascnt{conjecture}{theorem}
\newaliascnt{proposition}{theorem}

\newtheorem{problem}{Problem}

\usepackage{color}

\newcommand{\comment}[1]{\iffalse #1 \fi}

\title{Binary Split Categorical feature with Mean Absolute Error Criteria in CART}
\author{
    Peng Yu\textsuperscript{\rm 1,2,4},
    Yike Chen\textsuperscript{\rm 1},
    Chao Xu\textsuperscript{\rm 1}\thanks{Corresponding author.},
    Albert Bifet\textsuperscript{\rm 2},
    Jesse Read\textsuperscript{\rm 3}
}
\affiliations{
    \textsuperscript{\rm 1}University of Electronic Science and Technology of China\\
    \textsuperscript{\rm 2}Télécom Paris, Institut Polytechnique de Paris\\
    \textsuperscript{\rm 3}École Polytechnique, Institut Polytechnique de Paris\\
    \textsuperscript{\rm 4}Shopify\\
    peng.yu@shopify.com, cyike9982@gmail.com, cxu@uestc.edu.cn, albert.bifet@telecom-paris.fr, jesse.read@polytechnique.edu
}

\usepackage{bibentry}
\begin{document}

\maketitle

\begin{abstract}

In the context of the Classification and Regression Trees (CART) algorithm, the efficient splitting of categorical features using standard criteria like GINI and Entropy is well-established. However, using the Mean Absolute Error (MAE) criterion for categorical features has traditionally relied on various numerical encoding methods. This paper demonstrates that unsupervised numerical encoding methods are not viable for the MAE criteria. Furthermore, we present a novel and efficient splitting algorithm that addresses the challenges of handling categorical features with the MAE criterion. Our findings underscore the limitations of existing approaches and offer a promising solution to enhance the handling of categorical data in CART algorithms.

\end{abstract}

\section{Introduction}

The CART family of algorithms (random forest, gradient boosting tree) is well-known for its top performance on tabular data. Real-world tabular data often contains not only numerical but also categorical features. The CART algorithm recursively partitions the input dataset with a binary split optimization step and terminates when reaching a minimum number of instances. While traditional machine learning models only work with numerical data, the CART family of algorithms can process categorical features directly. This flexibility is because the binary split optimization step in the CART algorithm only requires feature data types that allow for different subsets.

The binary split step is recognized as a major bottleneck regarding the computational efficiency of tree learning algorithms \cite{CATLETT1991596}. Specifically, when processing categorical features, the associated discrete set topology can result in an exponential search space for binary splits. As a result, various numerical encoding methods have been developed to address this limitation. Consequently, many popular tree-based machine learning software packages (such as XGBoost \cite{chen2016xgboost}, LightGBM \cite{NIPS2017_6907}, and Catboost \cite{prokhorenkova2018catboost}) only support numerical data or include automatic numerical encoding methods for categorical data. On the other hand, only a subset of splitting criteria (such as mean squared error and Gini impurity) have optimally guaranteed numerical encodings for categorical data \cite{hastie01statisticallearning}. The splitting criterion MAE (Mean Absolute Error) lacks a proven optimal numerical encoding. MAE is more robust when dealing with outliers and skewed distributions and is widely adopted in various statistical domains. The most successful and practical numerical encoding method for this criterion is a median-based heuristic numerical encoding. This heuristic has been implemented in \texttt{scikit-learn}, and takes $O(n^2)$ time, where $n$ is the dataset size \cite{githubissue}. Unfortunately, $O(n^2)$ running time is too slow for practical purposes. The community suggests using subsampling to avoid the running time issue at the cost of finding a worse split, which is the standard in \texttt{LightGBM}. % Nonetheless, a simple proof of its non-optimality remained elusive \cnote{what does this mean at all?}, as verifying the optimality of large random examples by hand is impractical \cite{torgo1999inductive}.

\paragraph{Our Contributions}

We are motivated by two open problems: whether there is a numerical encoding that works for MAE, and if not, does there exist a fast exact algorithm for binary split through MAE?
\begin{enumerate}
    \item We prove that no unsupervised numerical encoding method is optimal for MAE, and show a median heuristic could be twice the optimal. While the proof itself is relatively straightforward, the significance of this result is reflected in the substantial effort invested in seeking the optimal numerical encoding method. For instance, dozens of unsupervised numerical encoding methods are under development \cite{willmcginnis}. 
    \item We develop an \emph{exact} and completely new algorithm to solve the binary split of categorical features with the MAE criterion in $O(n \log n + k\log k \log n)$ time without using numerical encoding, where $n$ is the number of data points and $k$ is the number of categories. The new algorithm is faster than the current heuristics, and also gives the exact result. So it both handles real-world size data without subsampling and is completely optimal. The new algorithm may also hold independent interest, as it solves the \emph{unimodal cost $2$-median problem},
    which generalizes various problems studied in computational geometry literature.  
\end{enumerate}

\section{Preliminaries}

In regression tree learning, during a node split computation, the goal is to find a binary partition $S, S^c$ of $Y \in \mathbb{R}^n$, based on some feature $X$. Here, $S^c$ is the complement of $S$ with respect to $Y$. When $X$ is categorical, the goal can be further simplified as finding a binary partition of $\mathcal{Y} = \{ Y_1, Y_2, \ldots Y_k \}$, where $Y_i \subseteq \mathbb{R}$ is the subset of the target data points with category $i$ in feature $X$. We assume the collection of data points in $\mathcal{Y}$ is a set, and they are all disjoint. This is only for the benefit of exposition. Everything would still hold with proper definition if repeated data points are allowed. One easy way is to assume that if there are copies of the same element, we just perturbed each one of them by an infinitesimal amount, so all data points are unique.

Depending on the splitting criteria, the partition must maximize or minimize an objective function. For a given set $\mathcal{S} \subseteq \mathcal{Y}$, the MAE is defined as 
\begin{equation}
     \MAE(\mathcal{S})  = \min_{a\in \R} \sum_{x \in \cup \mathcal{S}} | x - a |.
\end{equation}
The $a$ achieving the minimum is $\textbf{M}(\cup \mathcal{S})$, where $\textbf{M}(\cdot)$ is the median of the input. Define the error of a split $\mathcal{S}$ and $\mathcal{S}^c$ as $\lambda(\mathcal{S})$, which is the sum of their MAE. 

\begin{equation}
\label{eq:o}
  \lambda(\mathcal{S}) :=  \MAE(\mathcal{S})+\MAE(\mathcal{S}^c).
\end{equation}

When using MAE as the splitting criterion, the objective value would be the minimum over all splits: 
\begin{equation}
\label{eq:o}
  \lambda :=  \min_{\mathcal{S} \subset  \mathcal{Y}, \mathcal{S} \neq \emptyset } \lambda(\mathcal{S}).
\end{equation}

% Other common splitting criteria including mean square error [expand].

The problem of computing the split $S, S^c$ that achieves the minimum $\lambda(S)$ is referred to as the \emph{MAE split problem}.

To solve for the minimum, one can enumerate all subsets of $\mathcal{Y}$, resulting in an undesirable $O(2^k)$ search space. On the other hand, the community has developed numerous heuristic numerical encoding methods. 

\paragraph{Unsupervised Numerical Encoding}

    \begin{figure}[h]
    \centering
        \includegraphics[width=\linewidth]{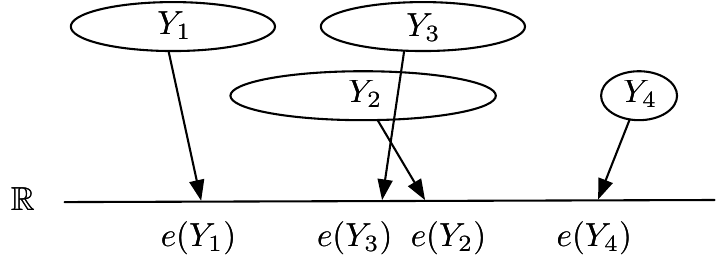}
        \caption{A example where $\mathcal{Y}=\{Y_1,Y_2,Y_3,Y_4\}$, and the encoding function $e$ maps elements in $\mathcal{Y}$ to $\mathbb{R}$. There are $5$ downward closed sets, and $3$ of them splits, so $D(\mathcal{Y},e) = \{\{Y_1\},\{Y_1,Y_3\},\{Y_1,Y_3,Y_2\}\}$.}
        \label{fig:encoding}
    \end{figure}

Rather than enumerating all subsets, one can employ a numeric encoding/set function $e: 2^{\mathbb{R}} \to \mathbb{R}$ to establish an ordering $\preceq_{e}$ and select sets based on this ordering. 

A numerical encoding function is considered unsupervised if specified independently of the data, meaning it is determined without observing the input.

The numerical encoding provides a natural ordering over the elements of $\mathcal{Y}$, where $A \preceq_{e} B$ if $e(A) \leq e(B)$. We define the \emph{downward closed sets} of $\mathcal{Y}$ as $D'(\mathcal{Y},e)$, which consists of sets of the form $\{A \mid e(A) \leq x, A \in \mathcal{Y}\}$ for some $x$. 

Numerical encoding has been used as a heuristic to identify the optimal partition within the downward closed sets. Specifically, let $D(\mathcal{Y},e)=D'(\mathcal{Y},e)\setminus \{\mathcal{Y}, \emptyset\}$, which are downward closed sets that splits $\mathcal{Y}$ into two non-empty sets. See \Cref{fig:encoding} for example. Our goal is to find $\mathcal{S} \in D(\mathcal{Y},e)$ that minimizes the objective $\lambda$. This modified problem leads to a more favorable $O(k)$ search space, although optimality is not guaranteed. 

For instance, the median heuristic employs the encoding function $e=\textbf{M}$ by arranging sets based on their medians. It then enumerates the downward closed sets and their complements as potential splits. Consequently, this approach yields only $k-1$ potential splits, which can be tested one by one and return the optimum.

    %\cnote{Give a simple example.}
    %\cnote{Define the median heuristic somewhere later, so people know what it looks like, and recall it is a spcial case of numerical encoding.}

    %%%%%%%%%%%%%%%%%%%%%%%%%%%%%%%%%%%%%%%%%%%%%%%%%%%%%%%%%%%%%%%%
    
    \paragraph{Piecewise-linear functions}
    We also introduce a few useful functions and their properties. 
    A function $f:\R\to \R$ is \emph{unimodal}, if there exists a $c$ such that for any $x\leq y\leq c$, we have $f(x)\geq f(y)$, and for $c\leq x\leq y$, $f(x)\leq f(y)$. 
    
    A function $h:\R\times \R \to \R$ is \emph{totally monotone}, if for any $x_1\leq x_2\leq y_1\leq y_2$, $h(x_1,y_1)\geq h(x_1,y_2)$ then $h(x_2,y_1)\geq h(x_2,y_2)$. A positive linear combination of totally monotone functions is totally monotone. A matrix is totally monotone if the function $h(i,j) = M_{i,j}$ is totally monotone. The main property of a totally monotone matrix is the index for row minimum is non-decreasing. That is, if $M_{i,a_i}$ is the minimum value of the $i$th row, then we have $a_1\leq a_2 \leq \ldots \leq a_n$ \cite{Park99}.
    
    We use $B(f)$ to denote the set of breakpoints for a piecewise-linear function. 
    Recall the median function $\median(S)$ is an element $y\in S$, such that $\sum_{x\in S} |x-y|$ is minimized.

  %  \begin{figure}
  %      \begin{tabular}{ c c c }
  %       Splitting Criteria & Numerical Encoding & optimal \\ 
  %       \hline
  %       MSE  & mean  & $\checkmark$ \\  
  %       Gini impurity & single class percentage & $\checkmark$\\
  %       MAE & one hot & $\times$\\
  %       MAE & median & $\times$\\
  %      \end{tabular}
  %      \caption{Comparison of unsupervised encoding for different criteria.}
  %      \label{tab:encoding}
  %  \end{figure}

\begin{table}[h!]
\centering
\caption{Comparison of unsupervised encoding for different criteria.}
\label{tab:encoding}
\begin{tabular}{l l c}
\hline
\textbf{Splitting Criteria} & \textbf{Numerical Encoding} & \textbf{Optimal} \\
\hline
MSE             & Mean                    & $\checkmark$ \\
Gini impurity   & Single class percentage  & $\checkmark$ \\
MAE             & One-hot                 & $\times$ \\
MAE             & Median                  & $\times$ \\
\hline
\end{tabular}
\end{table}

    \section{Numerical encoding and median heuristic}
    
    Is there a numerical encoding that can be used to find the optimal binary split for MAE? 
    Specifically, is there an encoding function $e$ such that the following equality holds: $\lambda$ is equal to $\min_{\mathcal{S}\in D(\mathcal{Y},e)} \lambda(S)$?

    \Cref{tab:encoding} shows target mean-based numerical encoding for categorical features has been proven optimal in decision tree regression with mean squared error (MSE) in \cite{breiman1984classification}, the same heuristic does not work with MAE. 

    Still, this does not rule out the existence of other unsupervised numerical encodings that work for MAE. Unfortunately, we prove that such encoding cannot exist.

    Suppose a unique optimal partition of a dataset minimizes the MAE. In that case, any encoding that works for MAE must have the encoding of all elements in one partition strictly smaller than or greater than the encoding of the other partition. Formally, let $\{\mathcal{A},\mathcal{B}\}$ forms the unique optimum partition of dataset $\mathcal{Y}$, and $e$ is a encoding that \emph{works} for MAE, then either $e(A)<e(B)$ for all $A\in \mathcal{A}$ and $B\in \mathcal{B}$, or $e(A)>e(B)$ for all $A\in \mathcal{A}$ and $B\in \mathcal{B}$. If the encoding $e$ works for MAE, then the optimal partition is in $D(\mathcal{Y}, e)$.
    
    \begin{theorem}
    No numerical encoding function works for binary split with MAE splitting criteria.
    \end{theorem}
    \begin{proof}
    Assume such an encoding $e$ exists and prove by contradiction via constructing a counter-example. Let $\epsilon>0$ be some small and fixed value, say $0.01$. 
    
    Let $a_1=0,a_2=2,a_3=3,a_4=5$. We define $A_i = \{a_i-\epsilon, a_i, a_i+\epsilon\}$, $A_i'=\{ a_i-\epsilon, a_i+\epsilon, a_1\}$ if $i\in\{3,4\}$, otherwise $A_i' = \{a_i-\epsilon, a_i+\epsilon, a_4\}$.
    \begin{enumerate}
    \item The unique optimum partition of $\{ A_1, A_1', A_4, A_4' \}$ is $\{ A_1, A_1' \}, \{A_4, A_4' \}$, without loss of generality, let $e(A_1)<e(A_4)$.
    \item The unique optimum partition of $\{ A_2, A_1', A_3, A_4' \}$ is $\{ A_2, A_1' \}, \{A_3, A_4' \}$, hence $e(A_1') < e(A_4')$.
    \item The unique optimum partition of $\{ A_2, A_2', A_3, A_3' \}$ is $\{ A_2, A_2' \}, \{A_3, A_3' \}$, hence $e(A_2) < e(A_3)$.
    \item The unique optimum partition of $\{ A_1, A_2', A_3', A_4 \}$ is $\{ A_1, A_3' \}, \{A_2',A_4\}$, hence $e(A_4)  < e(A_1)$, a contradiction.
    \end{enumerate}
    \end{proof}
    
    Now we know that any encoding-based heuristic would not give the correct result, but maybe it can still give a good enough result. To answer this, we examine the limitations of the most widely used encoding-based heuristic, the median heuristic.
    
    Empirically, it has been shown that the median numerical encoding works most of the time for MAE splitting criteria \cite{torgo1999inductive}. The conclusion is made based on experiments on limited datasets. The median encoding result in sub-optimal splits was only observed through some rare, randomly generated datasets.  
    
    However, we design an input so that the median heuristic is almost twice as bad as the optimum. Let $n$ be an even integer. Consider 4 collections of data points $Y_0,Y_1,Y_2,Y_3$. $Y_0,Y_1$ consists of $n$ copies of $0$ and $1$, respectively. Let $Y_2$ consist of $n/2$ copies of $0$, $n/2+1$ copies of $0.5+\epsilon$, and $Y_3$ consists of $n/2$ copies of $1$ and $n/2+1$ copies of $0.5-\epsilon$. Using the median heuristic, the potential points to split are $0, 0.5-\epsilon, 0.5+\epsilon, 1$. Observe that no matter which one is chosen, the output of the median heuristic would give a solution of value $n+2\epsilon$. The actual optimal would give a value around $n/2+1$. 

    \section{Methodology}
    Knowing that no unsupervised numerical encoding $e$ works for MAE, there might still be efficient algorithms that give the exact solution and don't use any encoding. In the following section, we propose such an algorithm. The algorithm is a fairly complicated divide-and-conquer that uses tools in computational geometry.

 We first transform the MAE split problem into a more manageable version, where instead of optimization of subsets (which can be as large as $2^k$), it becomes optimizing over $O(n^2)$ points. %, which avoids the median function. 
   \begin{lemma}\label{prop:transform-equivalence}
Let $\mathcal{Y}$ be a family of disjoint sets. Then
\[
\lambda \;=\; \min_{\emptyset \subsetneq \mathcal{S}\subsetneq \mathcal{Y}} \bigl(\MAE(\mathcal{S}) \;+\; \MAE(\mathcal{S}^c)\bigr)
\]
can be equivalently written as
\[
\lambda \;=\; \min_{a,b \,\in\, \mathbb{R}}
\;\sum_{S \,\in\, \mathcal{Y}} \min\Bigl(\sum_{i \,\in\, S} \lvert i - a\rvert,\;\sum_{j \,\in\, S} \lvert j - b\rvert\Bigr).
\]
\end{lemma}

\begin{proof}
By definition, 
\[
\MAE(\mathcal{S}) \;=\; \min_{a \,\in\, \mathbb{R}}\,\sum_{i \in \bigcup \mathcal{S}} |i - a|.
\]
Hence, for any subset $\mathcal{S} \subseteq \mathcal{Y}$ with complement $\mathcal{S}^c$, we have
\[
\MAE(\mathcal{S}) + \MAE(\mathcal{S}^c)
\;=\;
\min_{a \in \mathbb{R}} \sum_{i \in \bigcup \mathcal{S}} \lvert i - a\rvert
\;+\;
\min_{b \in \mathbb{R}} \sum_{j \in \bigcup \mathcal{S}^c} \lvert j - b\rvert.
\]
Note that $\bigcup \mathcal{S}$ and $\bigcup \mathcal{S}^c$ partition all elements in $\bigcup \mathcal{Y}$. 
We may reorganize the sums \emph{set by set} (i.e., over each $S \in \mathcal{Y}$), introducing two real parameters $a$ and $b$. 
Gathering the terms for each $S$ and observing that the choice 
\[
\min\Bigl(\sum_{i \in S}|i-a|,\;\sum_{j \in S}|j-b|\Bigr)
\]
corresponds to placing $S$ into $\mathcal{S}$ or its complement $\mathcal{S}^c$, respectively, yields
\[
\begin{aligned}
  \min_{\emptyset \subsetneq \mathcal{S} \subsetneq \mathcal{Y}} & \bigl(\MAE(\mathcal{S})+\MAE(\mathcal{S}^c)\bigr) \\
  &= \min_{a,b \,\in\, \mathbb{R}} \sum_{S \,\in\, \mathcal{Y}} 
    \min\Bigl(\sum_{i \in S} |i - a|,\;\sum_{j \in S} |j - b|\Bigr).
\end{aligned}
\]

This confirms the claimed equivalence, completing the proof.
\end{proof}

%    We first transform the MAE split problem into a more manageable version, which avoids the median function. 
%    Indeed, we can transform Eq~\ref{eq:o} into the following. 
%    \begin{align*}
%        \lambda =&  \min_{\emptyset \subsetneq \mathcal{S} \subsetneq \mathcal{Y} } \MAE(\mathcal{S}) +  \MAE(\mathcal{S}^c)\\
%      =&\min_{\emptyset \subsetneq \mathcal{S} \subsetneq \mathcal{Y}} \min_{a\in \R}\sum_{i \in \cup \mathcal{S}} | i - a| + \min_{b\in \R}\sum_{j \in \cup \mathcal{S}^c} | j - b|\\
%      =&\min_{\emptyset \subsetneq \mathcal{S} \subsetneq \mathcal{Y},a,b\in \R} \sum_{i \in \cup \mathcal{S}} | i - a| + \sum_{j \in \cup \mathcal{S}^c} | j - b|\\
%      =&\min_{a,b\in \R} \min_{\emptyset \subsetneq \mathcal{S} \subsetneq \mathcal{Y}} \sum_{i \in \cup \mathcal{S}} | i - a| + \sum_{j \in \cup \mathcal{S}^c} | j - b|\\
%      =&\min_{a,b\in \R} \min_{\emptyset \subsetneq \mathcal{S} \subsetneq \mathcal{Y}} \sum_{S\in \mathcal{S}} \sum_{i\in S} | i - a| + \sum_{S\in \mathcal{S}^c} \sum_{j \in S} | j - b|\\
%      =&\min_{a, b \in \mathbb{R}} \sum_{S \in \mathcal{Y}} \min(\sum_{i \in S} | i-a|, \sum_{j \in S} | j-b|)
%    \end{align*}

    After the transformation, we consider the following optimization problem.

    \begin{problem}[Median split problem]\label{prob:ms}
        Given $\mathcal{Y}$, a family of subsets of $\R$, find $a,b\in \mathbb{R}$, such that $\sum_{S \in \mathcal{Y}} \min(\sum_{i \in S} | i-a|, \sum_{j \in S} | j-b|)$ is minimized.
    \end{problem}

    We define a few more substitutions to simplify (and at the same time, generalize) the problem even further.
    
    Define $f_S(x) = \sum_{y \in S} |y - x|$.
    Observe that we try to optimize $g(a,b) = \sum_{S\in \mathcal{Y}} \min \{ f_S(a), f_S(b)\}$. Note that each $f_S$ is piecewise-linear and unimodal. Indeed, let $c = \median(S)$, when $x <c$, $f_S$ is monotonically decreasing and when $x >c$, $f_S$ is monotonically increasing. Hence, $f_S$ is unimodal. 
    We use this property to design a much faster algorithm. To this end, we introduce a much more general problem, the Unimodal Cost 2-Median problem (UC2M).
    
    \begin{problem}[Unimodal Cost 2-Median (UC2M)]\label{prob:uni}
        Let $f_1,\ldots,f_k:\R\to \R$ be $k$ piecewise-linear unimodal functions with a total of $n$ breakpoints. Let $g(a,b) = \sum_{i=1}^k \min\{f_i(a),f_i(b)\}$. Find $a,b\in \R$ such that it minimizes $g$.
    \end{problem}

    \begin{theorem}
    \Cref{prob:ms} reduces to \Cref{prob:uni} in linear time.
    \end{theorem}
    \begin{proof}
        Let $f_S = \sum_{i \in S} |i - x|$.
        Let the input to the \Cref{prob:ms} be $\mathcal{Y} = \{S_1,\ldots,S_k\}$.  This reduces to \Cref{prob:uni} with input functions $f_{S_1}, \ldots, f_{S_k}$.
    \end{proof}

    See \Cref{alg:reduction} for the entire reduction assuming inputs are data points of pairs $(r,c)$, where $r$ is the value and $c$ is the category of the feature. Note for simplicity of presentation, we only compute the value instead of the actual split, but it is easy to extend it to return the entire solution.
    
    \begin{figure}[h]
        \centering
        \begin{algorithm}
            \textsc{BinaryMAESplit}($data$):\+
            \\  $C \gets$ list of categories of the feature
            \\  for $c\in C$\+
            \\    $G\gets \emptyset$
            \\    for $(r,c) \in data$\+
            \\        add function $x\mapsto |r-x|$ to $G$\-
            \\    $f_c \gets \sum_{g\in G} g$\-
            \\  return \textsc{Unimodal2Median}($\{f_c | c\in C\}$)\-
        \end{algorithm}
        \caption{Find the optimum value of the MAE criteria.}
        \label{alg:reduction}
    \end{figure}
    
    Therefore, we shift gears and try to solve \Cref{prob:uni} in the remainder of the paper. 
    
    \section{Algorithm for Unimodal Cost $2$-Median}
    
    UC2M is related to various classical 2-median problems in 1D, where there is a cost that is a function of the \emph{distance} to the center \cite{HASSIN1991395,danny}. Those problems can be seen as a \emph{symmetric} unimodal cost $2$-median problem. A function $f:\R\to\R$ is \emph{symmetric}, if there exists some $c\in \R$, such that $f(c-x)=f(c+x)$ for all $x\in \R$. Previous techniques do not help with our problem, as our functions are not always symmetric. Hence, we have to design the algorithm from the ground up.
    
    For a piecewise-linear function $f$ of $n$ breakpoints, the representation consists of the sorted list of breakpoints $x_1,\ldots,x_n$, and the corresponding values $f(x_1),\ldots,f(x_n)$. Additionally, the initial slope and the final slope are also stored. Given $i$, one can find $x_i$, and evaluate $f$, the right slope of $f$, and the change of slope of $f$ at $x_i$, all in $O(1)$ time. If we are interested in finding the value of $f(x)$ by giving $x$ instead of any index, it takes $O(\log n)$ time by doing a binary search over the list and then interpolating adjacent breakpoints.

    \subsection{Properties of the problem}
    Let $f_1,\ldots,f_k$ be the input of \Cref{prob:uni}, and $g(x,y)=\sum_{i=1}^k \min\{f_i(x),f_i(y)\}$. Evaluate $g$ for all $x,y$ takes $O(k)$ time each if we look at $x,y$ in order. Hence, \Cref{prob:uni} has an $O(n^2k)$ time algorithm. 
    
    This algorithm is extremely naive, looking through all possible input pairs. One might guess that if we fix $a$, then the function $g_a(b) = g(a,b)$ is a unimodal function, and then a binary search-like procedure can be applied to find the minimum $b$ for $g_a$. Unfortunately, this is false; $ g_a$ can have many local minima. Fortunately, $g$ is a totally monotone function. 

    \begin{theorem}\label{thm:tm}
    Let $f:\R\to \R$ be an unimodal function. The function $g:\R^2\to \R$ defined as $g(x,y) = \min(f(x),f(y))$ if $x\leq y$, and $g(x,y)=\infty$ if $x>y$, is a totally monotone function.
    \end{theorem}
    \begin{proof}
    Consider for any $x_1\leq x_2$ and $y_1\leq y_2$.
    
    In order to show $h$ is totally monotone, we have to show that if $\min(f(x_1), f(y_1)) \leq \min(f(x_1), f(y_2))$, then $\min(f(x_2), f(y_1)) \leq \min(f(x_2), f(y_2))$.

    If we do not have $x_1 \leq x_2 \leq y_1 \leq y_2$, we will obtain an infinity case, and one can see the inequalities hold.

    Hence we assume  $x_1 \leq x_2 \leq y_1 \leq y_2$.
    
    There are two cases.
    
    Case 1. If $f(y_1) \leq f(y_2)$, then $\min\{f(x_2), f(y_1)\} \le \min\{f(x_2), f(y_2)\}$.
        
    Case 2. Otherwise, assume $f(y_1) > f(y_2)$. Because $f(y_1)> f(y_2)$ but $y_1\leq y_2$, so we must have $y_1$ is in the decreasing part of the function $f$. Therefore, we must have $f(x_1) \ge f(x_2) \ge f(y_1) > f(y_2)$. Hence, $f(y_1) = \min\{f(x_1), f(y_1)\} \leq \min\{f(x_1), f(y_2)\} = f(y_2) < f(y_1)$, a contradiction.
    \end{proof}
    
    By \Cref{thm:tm}, and the fact that the sum of totally monotone functions is totally monotone \cite{Park99}, the function we try to optimize in \Cref{prob:uni} is a totally monotone function.
    Let $x_1,\ldots,x_n$ be all the breakpoints of all the functions ordered from smallest to largest. Consider the matrix $M$, such that $M_{i,j} = g(x_i,x_j)$, where $x_i$ is the $i$th breakpoint of $g$, then $M$ is a totally monotone matrix. It is useful for us to consider the index-based version of the breakpoint instead of the breakpoint itself for ease of implementation.

    For a totally monotone matrix $M$, the SMAWK algorithm finds the row minima of each row of $M$ in $O(n)$ evaluations of entries in $M$ \cite{SMAWK}. Each evaluation takes $O(k\log k)$ time. Therefore, we obtain an $O(nk\log k)$ time algorithm. The SMAWK algorithm is known to use a minimum number of evaluations (up to a constant), so it seems there is no way to beat the current running time by much, as it is unlikely to do a single evaluation in less than $O(k)$ time in the worst case. 

    However, observe that evaluations are not all independent. After evaluating $M_{i,j}$, evaluating $M_{i+1,j}$ or $M_{i,j+1}$ would become easier, as the difference is only a single breakpoint, so the change in the function $g$ is easy to describe. Hence, if one can arrange the order of evaluation and compute the "difference" with a better data structure, a fast algorithm can be designed. In the next section, we show this is possible by \emph{massively speeding up} the average time of each evaluation at the cost of \emph{slightly increasing} the number of evaluations.
    
    % On the other hand, the median encoding method performs $k$ evaluations. If the same search type is performed, the median encoding method will exhibit a time complexity of $O(k^2\log k)$. But each search does not have to be independent. In practice, by arranging the order of evaluation with better data structure, $O(n\log n)$ time complexity is achieved. Drawing inspiration from this, the subsequent section presents an accelerated algorithm that marginally increases the evaluation count while improving individual evaluation speed.

    \subsection{Slowing down to speed up}

    Recall that we are interested in finding the $x$ and $y$ such that $g(x,y)$ is minimized, where $x, y$ are from a lattice grid and $g(x, y)$ evaluated on the grid results in a totally monotone matrix. 
    We can make the evaluation dependent on predecessors through an alternative divide-and-conquer algorithm other than the SMAWK. 
    This new divide-and-conquer algorithm will take a total of $O(n\log n)$ evaluations of the matrix, so a \emph{slow down} in the number of evaluations. However, a total \emph{speed up} is obtained by speeding up each individual evaluation.

    %, and a more formal idea can be seen in [cite the divide and conquer algorithm]. 

    We outline the idea as follows: 
    for a fixed $i$, let $j$ be the value that minimizes $M_{i,j}$. The optimum of the entire matrix must be either $M_{i,j}$, or of the form $M_{i',j'}$ where $i'<i$ and $j'\leq j$, or $i'>i$ and $j'\geq j$ \cite{Park99}. Hence, this gives us a natural divide-and-conquer algorithm: find the row minimum of the center row and recursively solve the new problem on the two smaller matrices. Observe the total number of evaluations of a $n\times m$ matrix would be $T(n, m) = T(n/2, m_1) + T(n/2, m_2) + O(m) = O(m\log n)$. Since, in our case, $n=m$, we get an algorithm taking $O(n\log n)$ evaluations. 

    Naively, this would give us an $O(nk\log n\log k)$ time algorithm, which is even worse than the SMAWK algorithm. However, instead of the number of evaluations, we can show that the \emph{running time} follows a similar recursion, and thus we obtain an $O(n\log n + k\log k \log n)$ time algorithm. 

    Naturally, we have to describe how to solve the two parts of the problem: Finding the row minimum and divide-and-conquer. 

    \subsection{Find the minimum over a single row}\label{sec:singlerow}

    Finding a minimum of a given row in the matrix $M$, is equivalent to answer the following question: Given functions $f_1,\ldots,f_k$, and an fixed index $a$, how to find $\min_{b} \sum_{j=1}^k \min(f_j(x_a),f_j(x_b))$ quickly?

    For a fixed $a$, define the \emph{active set} at index $b$ to be the set of function indices $j$, such that $f_j(x_a)> f_j(x_b)$. Let $A_1,\ldots, A_n$ be the sequence of active sets at $1,\ldots,n$, respectively. Each function $f_j$ moves out of the active set only once. That is, for an index $j\in [k]$, there exists an index $q_j$, such that for each $i\geq q_j$, we have $j\in A_i$, and $j\not \in A_i$ otherwise. 
    
    Define $f_A = \sum_{i\in A} f_i$. If we can quickly evaluate $f_A$ and update $A$, then we can quickly evaluate $g$. Indeed, in order to evaluate $g(x_a, x_b)$, the idea is to break it down into evaluating $f_{\bar{A}}(x_a) + f_{A}(x_b)$ where $A$ is the active set at index $b$.
    
    It's not hard to describe a data structure that maintains the value of $f_A(x_a)$ under updates of $A$ and $a$; this is precisely the evaluation data structure in \Cref{thm:time}. However, we must also decide which function has to be added or removed from $A$. This is done through a useful transformation. Let $f$ be an unimodal function with the local minimum at $c$, define $f^{\dagger}:\R\to \R$ to be $f^{\dagger}(x) = \max \{x' \mid f(x')\leq f(x)\}$ if $x\in (-\infty,c]$ and $-\infty$ otherwise. Intuitively, this means for any $x\leq c$, if we have $x\leq x'\leq f^{\dagger}(x)$, then we know $f(x')\leq f(x)$. Namely, one can quickly observe if $f(x)\leq f(x')$ by checking if $x'\leq f^{\dagger}(x)$. See \Cref{fig:dagger}. 

    If $f$ is a piecewise-linear unimodal function of $n$ breakpoints, then $f^{\dagger}$ is a piecewise-linear decreasing function in $(-\infty,c]$ of $n$ breakpoints and can be computed from $f$ in $O(n)$ time. We can find the value of $f^{\dagger}(x)$ in $O(\log n)$ time. Since $f^{\dagger}$ can be computed in linear time when $f$ is created, we always assume that $f^{\dagger}$ is computed when we use $f$. 

    Knowing $f_i^{\dagger}(x_a)$ for each $i$, then we know precisely when $i$ moves out of the active set: $i$ moves out of the active set when $x_b>f_i^{\dagger}(x_a)$ for the first time. See \Cref{alg:recursion} for implementation of $\textsc{RowMinima}$.

    \begin{theorem}\label{thm:rowminimatime}
        If the input is $k$ functions with a total of $n$ breakpoints, row minima can be found in $O(n+k\log k)$ time.
    \end{theorem}
    \begin{proof}
        We analyze the running time of $\textsc{RowMinima}$.
        Ordering the functions by evaluation of $f^\dagger_j(x_a)$ for each $j$, which takes $\sum_{j=1}^k n_j\log n_j = O(k\log \frac{n}{k})$ time, where $n_j$ is the number of breakpoints of $f_j$.
        Sorting the $k$ functions by their $\dagger$ value takes $O(k\log k)$ time.
        The linear scan takes $O(n)$ time. Hence, going through each breakpoint takes $O(n)$ time. Note $k = O(n)$, hence $O(k\log \frac{n}{k})=O(n)$. The total running time of RowMinima is $O(n+k\log k)$. 
    \end{proof}

    \begin{figure}[h]
    \centering
        \includegraphics[scale=0.4]{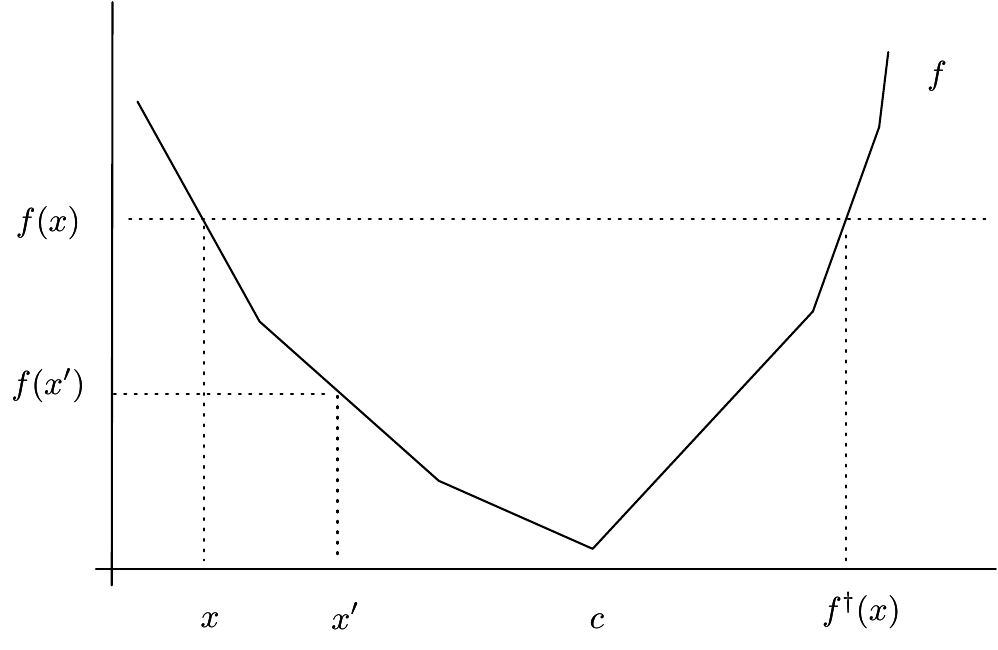}
        \caption{Intuition of the $\dagger$ transform.}
        \label{fig:dagger}
    \end{figure}

\begin{table*}[ht]
\centering
\caption{Comparison between LightGBM and our algorithm.}
\label{tab:experiment}
\resizebox{\textwidth}{!}{%
\begin{tabular}{l l r r r r r}
\hline
\textbf{Dataset} & \textbf{Feature} & \textbf{Data size} & \textbf{Categories} & \textbf{LightGBM relative Accuracy} & \textbf{LightGBM Time (s)} & \textbf{Our Time (s)} \\
\hline
Predict Droughts & TS & 19,300,680 & 7,588 & 0.9957 & 5.40698 & 0.948483 \\
                 & WS10M & 19,300,680 & 1,740 & 0.9991 & 5.64445 & 0.737965 \\
                 & QV2M & 19,300,680 & 2,210 & 0.9734 & 5.13713 & 0.843196 \\
                 & T2M\_RANGE & 19,300,680 & 3,029 & 0.9729 & 5.40553 & 0.847019 \\
delays\_zurich\_transport & windspeed\_avg & 5,465,575 & 66 & 0.9996 & 1.25839 & 0.283073 \\
                 & temp & 5,465,575 & 143 & 0.9984 & 1.26260 & 0.289960 \\
                 & stop\_id & 5,465,575 & 1,530 & 0.9766 & 1.29020 & 0.370960 \\
                 & time & 5,465,575 & 3,526 & 0.9870 & 1.23520 & 0.402581 \\
gpu\_kernel\_performance & MWG & 5,100,000 & 5 & 0.9234 & 2.25632 & 0.598805 \\
                 & MDIMC & 5,100,000 & 7 & 0.9903 & 2.33057 & 0.271983 \\
                 & NWG & 5,100,000 & 6 & 0.9583 & 2.36985 & 0.465714 \\
diamonds         & carat & 53,940 & 273 & 0.6193 & 0.021328 & 0.016229 \\
                 & table & 53,940 & 127 & 0.9829 & 0.022840 & 0.008226 \\
                 & x & 53,940 & 554 & 0.6213 & 0.028240 & 0.019635 \\
house\_sales     & sqft\_living & 21,613 & 1,038 & 0.8553 & 0.026370 & 0.009103 \\
                 & zipcode & 21,613 & 70 & 0.8322 & 0.016795 & 0.005777 \\
                 & sqft\_above & 21,613 & 946 & 0.8932 & 0.036526 & 0.008170 \\
wine             & fixed acidity & 1,143 & 91 & 0.8662 & 0.012946 & 0.000178 \\
                 & density & 1,143 & 388 & 0.7094 & 0.010698 & 0.000281 \\
                 & volatile acidity & 1,143 & 135 & 0.8067 & 0.005512 & 0.000193 \\
                 & citric acid & 1,143 & 77 & 0.8611 & 0.004860 & 0.000177 \\
boston           & ZN & 506 & 26 & 0.8926 & 0.005560 & 0.000262 \\
                 & INDUS & 506 & 76 & 0.7886 & 0.005856 & 0.000444 \\
                 & DIS & 506 & 412 & 0.7320 & 0.018420 & 0.000626 \\
\hline
\end{tabular}%
}
\end{table*}

    \subsection{Divide-and-conquer}\label{sec:singlerow}

    In the divide-and-conquer step, we split the problem into evaluations over $2$ smaller submatrices, which are almost disjoint.
    
    Observe that once we are searching for the optimum in a submatrix where the row index ranges from $a_{min}$ to $a_{max}$ and the column index ranges from $b_{min}$ to $b_{max}$. All the values of the functions outside this range are irrelevant. Hence, we can safely assume we process all the functions passed into the recursive call by removing the breakpoints outside the ranges. In practice, this is not explicit but is done through implicit bookkeeping. This allows us to bound the running time related to traversing the functions by $O((b_{max}-b_{min})+(a_{max}-a_{min}))$ inside each recursive call. See the recursive algorithm $\textsc{Optimum}$ in \Cref{alg:recursion}.

    The main observation is that the sequence of functions is also split into two subproblems. Let $M_{a,b}$ be the optimum value on the row $a$. We consider the functions in two classes, $L = \{i | f_i(x_a)\leq f_i(x_b)\}$, and $R = \{i | f_i(x_a)>f_i(x_b)\}$. The algorithm would be correct if we pass down all functions. However, when we pass the function to the left recursion, during the evaluating values $M_{a',b'}$ where $a'\in [a_{min}, a]$ and $b'\in [b_{\min},b]$, the contribution of the function $f_i$ where $i\in R$ is apparent: $\min(f_i(x_{a'}), f_i(x_{b'})) = f_i(x_{b'})$. Namely, $\sum_{i\in R} \min(f_i(x_{a'}),f_i(x_{b'})) = \sum_{i\in R} f_i(x_{b'})$. A similar result holds for right recursion. Therefore, there is no need to pass down all the functions; instead, we sum the functions and pass them down. This ensures the sequence of functions passed down is partitioned into two, each with one more function appended.

    \begin{figure}[ht]
    \centering
        \includegraphics[scale=0.11]{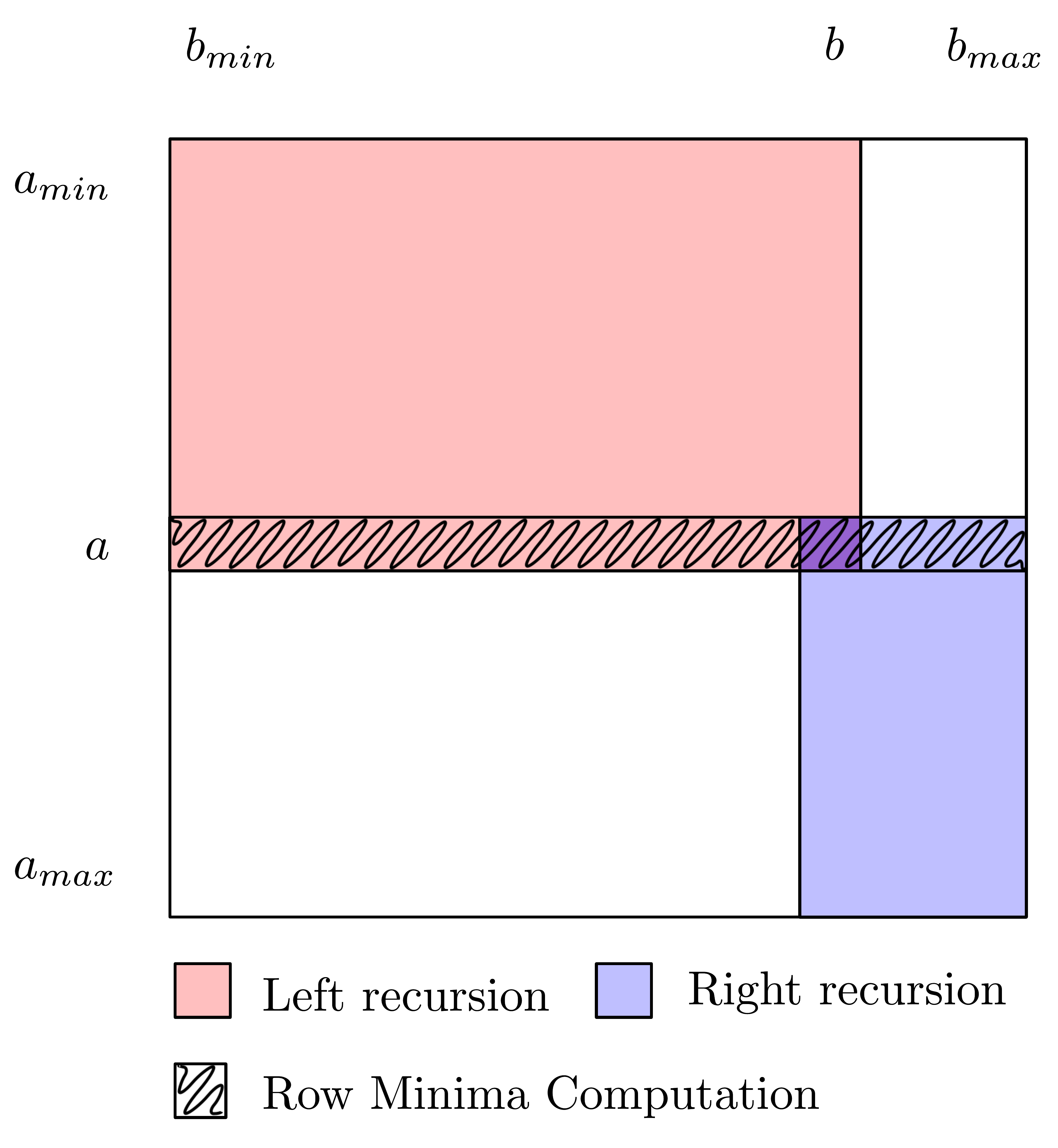}
        \caption{A demonstration of one step of the recursion algorithm.}
        \label{fig:recursion}
    \end{figure}

    \begin{figure}[ht]
        \centering
        \begin{algorithm}
            \textsc{Unimodal2Median}($f$):\+
            \\  Compute the global list of breakpoints $x_1,\ldots,x_n$
            \\  Compute the $\dagger$ transform for each function
            \\  return \textsc{Optimum}($1,n,1,n,f$)\-
            \\
            \\
            \textsc{Optimum}($a_{min},a_{max},b_{min},b_{max},f$):\+
            \\  $a\gets (a_{max}+a_{min})/2$
            \\  $v, b\gets \textsc{RowMinima}(a, b_{min}, b_{max}, f)$
            \\  if $a_{min}=a_{max}$:\+
            \\     return $v$\-
            \\  $L\gets \{i | f_i(x_a) \leq f_i(x_b)\}$
            \\  $R\gets \{i | f_i(x_a) > f_i(x_b)\}$
            \\  $f^L\gets \{f_i|i\in L\} \cup (\sum_{i\in R} f_i)$
            \\  $f^R\gets \{f_i|i\in R\}  \cup (\sum_{i\in L} f_i)$
            \\  $v_L \gets$ \textsc{Optimum}($a_{min},a,b_{min},b,f^L$)
            \\  $v_R \gets$ \textsc{Optimum}($a,a_{max},b,b_{max},f^R$)
            \\  return $\min(v_L,v_R,v)$\-
            \\
            \\
            \textsc{RowMinima}($a,b_{min},b_{max},f$):\+
            \\  $f_1,\ldots,f_k$ are renumbered such that $f^\dagger_j(x_a)\leq f^\dagger_{j+1}(x_a)$
            \\  $p\gets 1$
            \\  $A\gets \{1,\ldots,k\}$
            \\  for $i$ from $b_{min}$ to $b_{max}$\+
            \\      while $f^\dagger_p(x_a)<x_i$\+
            \\         $A\gets A\setminus \{p\}$
            \\         $p\gets p+1$\-
            \\      $v\gets f_{\bar{A}}(x_a) + f_{A}(x_i)$
            \\      if $v$ is smallest seen value\+
            \\         $b\gets i$\-\-
            \\  return $b$\-
        \end{algorithm}
        \caption{Algorithm for finding unimodal 2 medians. For each one of the procedures, the input $f$ is a sequence of piecewise-linear unimodal functions.}
        \label{alg:recursion}
    \end{figure}

   \begin{figure}[h]
    \centering
        \includegraphics[width=\linewidth]{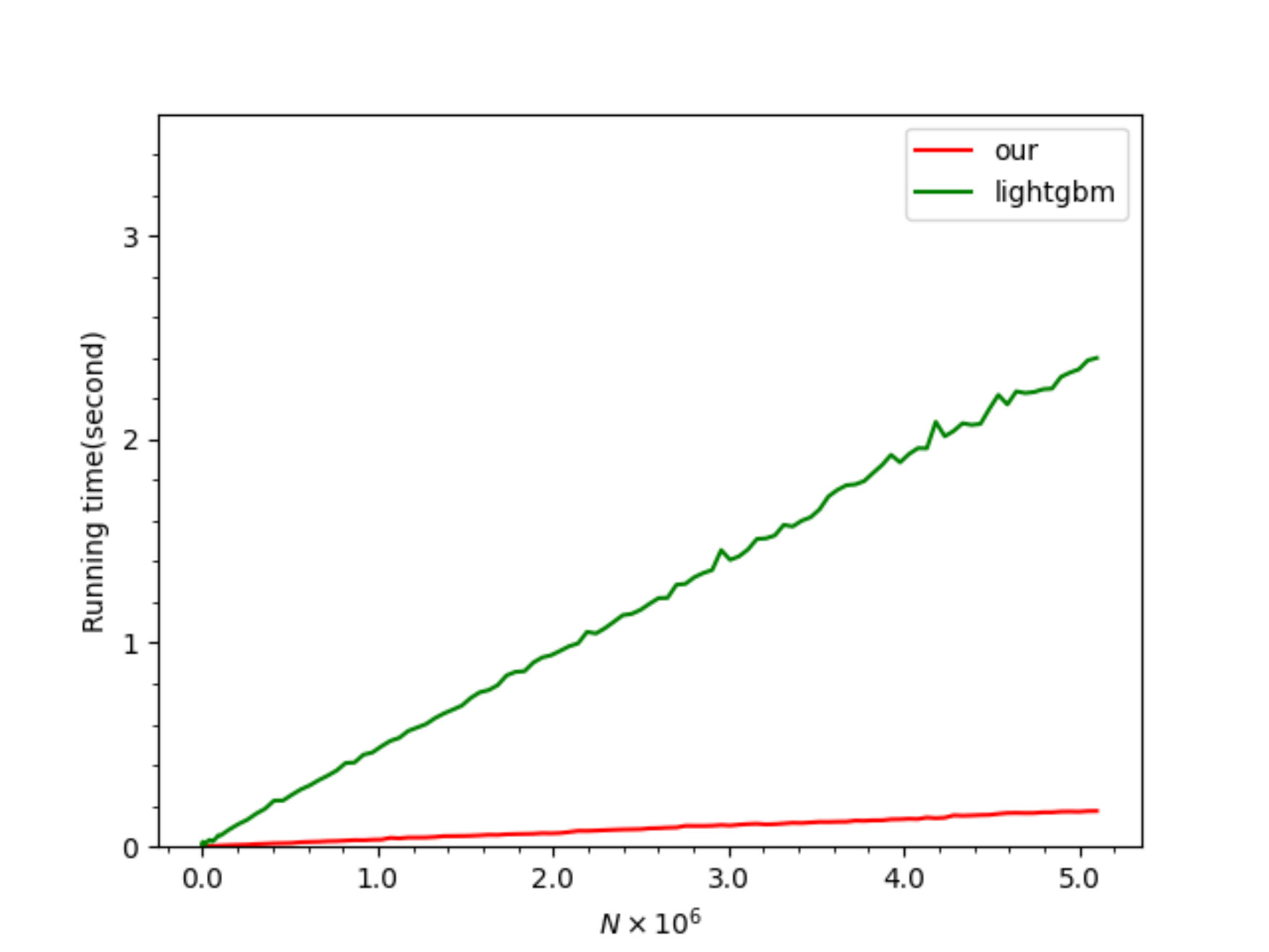}
        \caption{Running time in seconds vs number of data points}
        \label{fig:result}
    \end{figure}
    
    \begin{theorem}
        Finding the optimum of \Cref{prob:uni} for the input of $k$ function of a total $n$ breakpoints takes $O((n+k\log k)\log n)$ time.
    \end{theorem}
    \begin{proof}
        See appendix.
    \end{proof}

    Some additional optimization can be done that does not change the asymptotic worst-case running time but improves the running time in practice. For example, in the recursion, if at some point, the only function remaining is the function that came from a sum of original functions, then the recursion can stop earlier.

    \subsection{Data structure for piecewise-linear functions} \label{sec:evaluateDS}
    We describe a data structure over a set of piecewise-linear functions, and it can return the sum quickly. This data structure is required to implement $\textsc{RowMinima}$, where two dynamic sums of piecewise-linear functions need to be maintained, and also $\textsc{Optimum}$, where the sum of piecewise-linear functions has to be computed and passed down.

    Let $f_1,\ldots,f_k$ be piecewise-linear functions with a total of $n$ distinct breakpoints. Although our algorithms can handle the case when breakpoints are not distinct, describing them does not provide additional insights.
    
    Let $f_A = \sum_{i\in A} f_i$. Let a global set of points $x_1,\ldots,x_n$ contain all the breakpoints of each $f_i$. We aim to maintain a data structure that includes a set $A$ and an index $a$, enabling the fast evaluation of $f_A(x_a)$. 
    
    Formally, the data structure should have the following operations. 
    \begin{enumerate}
    \item \textsc{Initialize($f_1,\ldots,f_k$)}: Process the functions $f_1,\ldots,f_k$, and return a data structure for $f_{\emptyset}$ and $a=-1$.
    \item \textsc{Add}($i$): Update $A$ into $A\cup \{i\}$.
    \item \textsc{Remove}($i$): Update $A$ into $A\setminus \{i\}$.
    \item \textsc{Evaluate}(): Return $f_A(x_a)$.
    \item \textsc{Next}(): Update $a$ to $a+1$.
    \end{enumerate}
    
    Such a data structure is standard, but we sketch it here for completeness.
    
    \begin{theorem}\label{thm:time}
    Assuming all the breakpoints have been sorted, the data structure takes $O(n)$ time to construct, and any sequence of $O(n)$ queries takes $O(n)$ time.
    \end{theorem}
    \begin{proof}
    See appendix.
    \end{proof}

    \begin{theorem}
    \Cref{prob:uni} can be solved in $O((n+k\log k)\log n)$ time.
    \end{theorem}

    \begin{theorem}
    The median split problem on $k$ categories and $n$ data points can be solved in $O((n+k\log k)\log n)$ time.
    \end{theorem}

    In particular, if the number of categories is small with respect to the number of datapoints, then our running time is simply $O(n\log n)$.

\section{Experiments}

We implemented the algorithm in C++ and ran it on an 8-core AMD Ryzen 7 5800H processor with 16GB of RAM.\footnote{Code is available at \url{https://anonymous.4open.science/r/binary-split-F12B}.}
We selected several regression datasets from OpenML: two large datasets (more than 500k data points) \cite{openml_dataset_40753,openml_dataset_45662}, two medium-sized datasets (around 50k data points) \cite{openml_dataset_42225,openml_dataset_42731}, and two small datasets (about 1000 data points) \cite{winequality_2022,openml_dataset_531}. For each dataset, we attempted a binary split on each of its features. We used LightGBM, as it can handle large-scale data inputs, whereas scikit-learn often fails to complete in a reasonable amount of time.

We recorded the running time and MAE values of our algorithm and LightGBM, and calculated the relative accuracy (our result / LightGBM result) for each dataset, as shown in the table above (OpenML dataset IDs in parentheses).

To further evaluate performance under imbalanced data, we selected two datasets from Kaggle: Predict Droughts\cite{usdrought_2021} and Wine. Their results are also included in the table.

Our algorithm achieves exact MAE-optimal splits in all cases. Interestingly, LightGBM is sufficiently accurate in most instances. However, since our algorithm is consistently faster, there is no need to use subsampled data to reduce computation. For the imbalanced datasets, our method maintains higher accuracy than LightGBM and also demonstrates superior time efficiency. Furthermore, since our evaluation focuses only on single-level binary splits, we have reason to believe that, in multi-level decision trees, the cumulative error from heuristic methods like LightGBM may grow significantly as the depth increases.

We also examined running time asymptotics by creating subsets of the \texttt{gpu\_performance} dataset with feature \texttt{MDIMC} via sampling without replacement from 0.01 to 1.0 in 0.01 increments. We ran both our algorithm and LightGBM on each subset and compared the runtime. The result in \Cref{fig:result} confirms our method scales nearly linearly with input size.

\clearpage
\bibliography{main}

@phdthesis{Park99,
  title = {The {{Monge Array}}: {{An Abstraction}} and {{Its Applications}}},
  author = {Park, James Kimbrough},
  year = {1999},
  langid = {english},
  school = {Massachusetts Institute of Technology}
}

@misc{githubissue,
  title = {Trees with MAE criterion are slow to train},
author = {jiangfeng},
year = {2017},
  howpublished = {\url{https://github.com/scikit-learn/scikit-learn/issues/9626}},
  note = {Accessed: 2024-12-30}
}

@article{HASSIN1991395,
title = {Improved complexity bounds for location problems on the real line},
journal = {Operations Research Letters},
volume = {10},
number = {7},
pages = {395-402},
year = {1991},
issn = {0167-6377},
doi = {https://doi.org/10.1016/0167-6377(91)90041-M},
url = {https://www.sciencedirect.com/science/article/pii/016763779190041M},
author = {R. Hassin and A. Tamir}
}

@InProceedings{danny,
author="Chen, Danny Z.
and Wang, Haitao",
editor="Dehne, Frank
and Iacono, John
and Sack, J{\"o}rg-R{\"u}diger",
title="New Algorithms for 1-D Facility Location and Path Equipartition Problems",
booktitle="Algorithms and Data Structures",
year="2011",
publisher="Springer Berlin Heidelberg",
address="Berlin, Heidelberg",
pages="207--218",
isbn="978-3-642-22300-6"
}

@article{SMAWK,
	author = {Aggarwal, Alok and Klawe, Maria M. and Moran, Shlomo and Shor, Peter and Wilber, Robert},
	date = {1987/11/01},
	doi = {10.1007/BF01840359},
	id = {Aggarwal1987},
	isbn = {1432-0541},
	journal = {Algorithmica},
	number = {1},
	pages = {195--208},
	title = {Geometric applications of a matrix-searching algorithm},
	url = {https://doi.org/10.1007/BF01840359},
	volume = {2},
	year = {1987}}

@phdthesis{torgo1999inductive,
  title={Inductive learning of tree-based regression models},
  author={Torgo, Lu{\'\i}s Fernando Ra{\'\i}nho Alves},
  year={1999},
  school={Universidade do Porto. Reitoria}
}

@incollection{CATLETT1991596,
	address = {San Francisco (CA)},
	author = {Jason Catlett},
	booktitle = {Machine Learning Proceedings 1991},
	doi = {https://doi.org/10.1016/B978-1-55860-200-7.50121-5},
	editor = {Lawrence A. Birnbaum and Gregg C. Collins},
	isbn = {978-1-55860-200-7},
	pages = {596-599},
	publisher = {Morgan Kaufmann},
	title = {Mega induction: a Test Flight},
	url = {https://www.sciencedirect.com/science/article/pii/B9781558602007501215},
	year = {1991}}

@software{willmcginnis,
  author       = {Will McGinnis and
                  hbghhy and
                  Wenwu Tao and
                  andrethrill and
                  Chapman Siu and
                  Cameron Davison and
                  Nicholas Bollweg},
  title        = {{scikit-learn-contrib/categorical-encoding: Release 
                   for zenodo}},
  month        = jan,
  year         = 2018,
  publisher    = {Zenodo},
  version      = {1.2.6},
  doi          = {10.5281/zenodo.1157110},
  url          = {https://doi.org/10.5281/zenodo.1157110}
}

@inproceedings{chen2016xgboost,
  title={Xgboost: A scalable tree boosting system},
  author={Chen, Tianqi and Guestrin, Carlos},
  booktitle={Proceedings of the 22nd acm sigkdd international conference on knowledge discovery and data mining},
  pages={785--794},
  year={2016},
  organization={ACM}
}

@inproceedings{prokhorenkova2018catboost,
  title={CatBoost: unbiased boosting with categorical features},
  author={Prokhorenkova, Liudmila and Gusev, Gleb and Vorobev, Aleksandr and Dorogush, Anna Veronika and Gulin, Andrey},
  booktitle={Advances in Neural Information Processing Systems},
  pages={6638--6648},
  year={2018}
}

@article{NIPS2017_6907,
title = {LightGBM: A Highly Efficient Gradient Boosting Decision Tree},
author = {Ke, Guolin and Meng, Qi and Finley, Thomas and Wang, Taifeng and Chen, Wei and Ma, Weidong and Ye, Qiwei and Liu, Tie-Yan},
journal = {Advances in Neural Information Processing Systems 30},
editor = {I. Guyon and U. V. Luxburg and S. Bengio and H. Wallach and R. Fergus and S. Vishwanathan and R. Garnett},
pages = {3146--3154},
year = {2017},
publisher = {Curran Associates, Inc.},
url = {http://papers.nips.cc/paper/6907-lightgbm-a-highly-efficient-gradient-boosting-decision-tree.pdf}
}

@article{breiman1984classification,
  title={Classification and Regression Trees},
  author={Breiman, L. and Friedman, J. and Stone, C.J. and Olshen, R.A.},
  isbn={9780412048418},
  lccn={83019708},
  series={The Wadsworth and Brooks-Cole statistics-probability series},
  year={1984},
  publisher={Taylor \& Francis}
}

@book{hastie01statisticallearning,
  address = {New York, NY, USA},
  author = {Hastie, Trevor and Tibshirani, Robert and Friedman, Jerome},
  publisher = {Springer New York Inc.},
  series = {Springer Series in Statistics},
  title = {The Elements of Statistical Learning},
  year = 2001
}

@misc{openml_dataset_40753,
  author       = {OpenML},
  title        = {delays zurich transport dataset (v.1)},
  year         = {2017},
  howpublished = {\url{https://www.openml.org/d/40753}},
  note         = {Accessed: 2025-1-9}
}

@misc{openml_dataset_45662,
  author       = {OpenML},
  title        = {simulated sgemm gpu kernel performance dataset (v.2)},
  year         = {2023},
  howpublished = {\url{https://www.openml.org/d/45662}},
  note         = {Accessed: 2025-1-9}
}

@misc{openml_dataset_42225,
  author       = {OpenML},
  title        = {diamonds dataset (v.1)},
  year         = {2019},
  howpublished = {\url{https://www.openml.org/d/42225}},
  note         = {Accessed: 2025-1-8}
}

@misc{openml_dataset_42731,
  author       = {OpenML},
  title        = {house sales dataset (v.3)},
  year         = {2020},
  howpublished = {\url{https://www.openml.org/d/42731}},
  note         = {Accessed: 2025-1-9}
}

@misc{openml_dataset_531,
  author       = {OpenML},
  title        = {boston dataset (v.1)},
  year         = {2014},
  howpublished = {\url{https://www.openml.org/d/531}},
  note         = {Accessed: 2025-1-7}
}

@misc{usdrought_2021,
  title        = {Predict Droughts using Weather \& Soil Data},
  author       = {Christoph Minixhofer},
  year         = {2021},
  howpublished = {\url{https://www.kaggle.com/datasets/cdminix/us-drought-meteorological-data}},
  note         = {Accessed: 2025‑07‑26}
}

@misc{winequality_2022,
  title        = {Wine Quality Dataset},
  author       = {Kaggle},
  year         = {2022},
  howpublished = {\url{https://www.kaggle.com/datasets/yasserh/wine-quality-dataset}},
  note         = {Accessed: 2025‑07‑26}
}

\makeatletter
\@ifundefined{isChecklistMainFile}{
  % We are compiling a standalone document
  \newif\ifreproStandalone
  \reproStandalonetrue
}{
  % We are being \input into the main paper
  \newif\ifreproStandalone
  \reproStandalonefalse
}
\makeatother

\ifreproStandalone
\documentclass[letterpaper]{article}
\usepackage[submission]{aaai2026}
\setlength{\pdfpagewidth}{8.5in}
\setlength{\pdfpageheight}{11in}
\usepackage{times}
\usepackage{helvet}
\usepackage{courier}
\usepackage{xcolor}
\frenchspacing

\begin{document}
\fi
\setlength{\leftmargini}{20pt}
\makeatletter\def\@listi{\leftmargin\leftmargini \topsep .5em \parsep .5em \itemsep .5em}
\def\@listii{\leftmargin\leftmarginii \labelwidth\leftmarginii \advance\labelwidth-\labelsep \topsep .4em \parsep .4em \itemsep .4em}
\def\@listiii{\leftmargin\leftmarginiii \labelwidth\leftmarginiii \advance\labelwidth-\labelsep \topsep .4em \parsep .4em \itemsep .4em}\makeatother

\setcounter{secnumdepth}{0}
\renewcommand\thesubsection{\arabic{subsection}}
\renewcommand\labelenumi{\thesubsection.\arabic{enumi}}

\newcounter{checksubsection}
\newcounter{checkitem}[checksubsection]

\newcommand{\checksubsection}[1]{%
  \refstepcounter{checksubsection}%
  \paragraph{\arabic{checksubsection}. #1}%
  \setcounter{checkitem}{0}%
}

\newcommand{\checkitem}{%
  \refstepcounter{checkitem}%
  \item[\arabic{checksubsection}.\arabic{checkitem}.]%
}
\newcommand{\question}[2]{\normalcolor\checkitem #1 #2 \color{blue}}
\newcommand{\ifyespoints}[1]{\makebox[0pt][l]{\hspace{-15pt}\normalcolor #1}}

\section*{Reproducibility Checklist}

\vspace{1em}
\hrule
\vspace{1em}

\textbf{Instructions for Authors:}

This document outlines key aspects for assessing reproducibility. Please provide your input by editing this \texttt{.tex} file directly.

For each question (that applies), replace the ``Type your response here'' text with your answer.

\vspace{1em}
\noindent
\textbf{Example:} If a question appears as
\begin{center}
\noindent
\begin{minipage}{.9\linewidth}
\ttfamily\raggedright
\string\question \{Proofs of all novel claims are included\} \{(yes/partial/no)\} \\
Type your response here
\end{minipage}
\end{center}
you would change it to:
\begin{center}
\noindent
\begin{minipage}{.9\linewidth}
\ttfamily\raggedright
\string\question \{Proofs of all novel claims are included\} \{(yes/partial/no)\} \\
yes
\end{minipage}
\end{center}
Please make sure to:
\begin{itemize}\setlength{\itemsep}{.1em}
\item Replace ONLY the ``Type your response here'' text and nothing else.
\item Use one of the options listed for that question (e.g., \textbf{yes}, \textbf{no}, \textbf{partial}, or \textbf{NA}).
\item \textbf{Not} modify any other part of the \texttt{\string\question} command or any other lines in this document.\\
\end{itemize}

You can \texttt{\string\input} this .tex file right before \texttt{\string\end\{document\}} of your main file or compile it as a stand-alone document. Check the instructions on your conference's website to see if you will be asked to provide this checklist with your paper or separately.

\vspace{1em}
\hrule
\vspace{1em}

% The questions start here

\checksubsection{General Paper Structure}
\begin{itemize}

\question{Includes a conceptual outline and/or pseudocode description of AI methods introduced}{(yes/partial/no/NA)}
yes

\question{Clearly delineates statements that are opinions, hypothesis, and speculation from objective facts and results}{(yes/no)}
yes

\question{Provides well-marked pedagogical references for less-familiar readers to gain background necessary to replicate the paper}{(yes/no)}
yes

\end{itemize}
\checksubsection{Theoretical Contributions}
\begin{itemize}

\question{Does this paper make theoretical contributions?}{(yes/no)}
yes

	\ifyespoints{\vspace{1.2em}If yes, please address the following points:}
        \begin{itemize}
	
	\question{All assumptions and restrictions are stated clearly and formally}{(yes/partial/no)}
	yes

	\question{All novel claims are stated formally (e.g., in theorem statements)}{(yes/partial/no)}
	yes

	\question{Proofs of all novel claims are included}{(yes/partial/no)}
	yes

	\question{Proof sketches or intuitions are given for complex and/or novel results}{(yes/partial/no)}
	yes

	\question{Appropriate citations to theoretical tools used are given}{(yes/partial/no)}
	yes

	\question{All theoretical claims are demonstrated empirically to hold}{(yes/partial/no/NA)}
	partial

	\question{All experimental code used to eliminate or disprove claims is included}{(yes/no/NA)}
	NA
	
	\end{itemize}
\end{itemize}

\checksubsection{Dataset Usage}
\begin{itemize}

\question{Does this paper rely on one or more datasets?}{(yes/no)}
no

\ifyespoints{If yes, please address the following points:}
\begin{itemize}

	\question{A motivation is given for why the experiments are conducted on the selected datasets}{(yes/partial/no/NA)}
	Type your response here

	\question{All novel datasets introduced in this paper are included in a data appendix}{(yes/partial/no/NA)}
	Type your response here

	\question{All novel datasets introduced in this paper will be made publicly available upon publication of the paper with a license that allows free usage for research purposes}{(yes/partial/no/NA)}
	Type your response here

	\question{All datasets drawn from the existing literature (potentially including authors' own previously published work) are accompanied by appropriate citations}{(yes/no/NA)}
	Type your response here

	\question{All datasets drawn from the existing literature (potentially including authors' own previously published work) are publicly available}{(yes/partial/no/NA)}
	Type your response here

	\question{All datasets that are not publicly available are described in detail, with explanation why publicly available alternatives are not scientifically satisficing}{(yes/partial/no/NA)}
	Type your response here

\end{itemize}
\end{itemize}

\checksubsection{Computational Experiments}
\begin{itemize}

\question{Does this paper include computational experiments?}{(yes/no)}
yes

\ifyespoints{If yes, please address the following points:}
\begin{itemize}

	\question{This paper states the number and range of values tried per (hyper-) parameter during development of the paper, along with the criterion used for selecting the final parameter setting}{(yes/partial/no/NA)}
	NA

	\question{Any code required for pre-processing data is included in the appendix}{(yes/partial/no)}
	yes

	\question{All source code required for conducting and analyzing the experiments is included in a code appendix}{(yes/partial/no)}
	yes

	\question{All source code required for conducting and analyzing the experiments will be made publicly available upon publication of the paper with a license that allows free usage for research purposes}{(yes/partial/no)}
	yes
        
	\question{All source code implementing new methods have comments detailing the implementation, with references to the paper where each step comes from}{(yes/partial/no)}
	partial

	\question{If an algorithm depends on randomness, then the method used for setting seeds is described in a way sufficient to allow replication of results}{(yes/partial/no/NA)}
	NA

	\question{This paper specifies the computing infrastructure used for running experiments (hardware and software), including GPU/CPU models; amount of memory; operating system; names and versions of relevant software libraries and frameworks}{(yes/partial/no)}
	yes

	\question{This paper formally describes evaluation metrics used and explains the motivation for choosing these metrics}{(yes/partial/no)}
	yes

	\question{This paper states the number of algorithm runs used to compute each reported result}{(yes/no)}
	yes

	\question{Analysis of experiments goes beyond single-dimensional summaries of performance (e.g., average; median) to include measures of variation, confidence, or other distributional information}{(yes/no)}
	no

	\question{The significance of any improvement or decrease in performance is judged using appropriate statistical tests (e.g., Wilcoxon signed-rank)}{(yes/partial/no)}
	no

	\question{This paper lists all final (hyper-)parameters used for each model/algorithm in the paper’s experiments}{(yes/partial/no/NA)}
	NA

\end{itemize}
\end{itemize}
\ifreproStandalone
\end{document}
\fi

\end{document}